\documentclass[11pt]{article}

\usepackage{newtxtext}
\usepackage[left=1.25in, top=1in, bottom=1in, right=1.25in]{geometry}

\usepackage[authoryear]{natbib}
\usepackage{graphicx}
\usepackage[utf8]{inputenc}
\usepackage[T1]{fontenc}
\usepackage[colorlinks=true,citecolor=blue,linkcolor=blue,urlcolor=blue,hypertexnames=false]{hyperref}
\usepackage{url}
\usepackage{booktabs}
\usepackage{amsfonts}
\usepackage{amssymb}
\usepackage{amsmath}
\usepackage{amsthm}
\usepackage{mathtools}
\usepackage{bm}
\usepackage{nicefrac}
\usepackage{microtype}
\usepackage{xcolor}
\usepackage{tcolorbox}
\usepackage{cleveref}
\usepackage{wrapfig}
\usepackage{pifont}
\usepackage{algorithm}
\usepackage{algorithmic}

\usepackage{enumitem}
\usepackage{float}
\usepackage{placeins}

\graphicspath{{fig/}}

\usepackage[textsize=tiny]{todonotes}

\theoremstyle{plain}
\newtheorem{theorem}{Theorem}[section]

\newtheorem{lemma}[theorem]{Lemma}
\newtheorem{corollary}[theorem]{Corollary}

\theoremstyle{definition}
\newtheorem{assumption}[theorem]{Assumption}

\crefname{assumption}{assumption}{assumptions}
\Crefname{assumption}{Assumption}{Assumptions}

\DeclareMathOperator{\msign}{msign}

\newcommand{\R}{\mathbb{R}}
\newcommand{\E}{\mathbb{E}}
\newcommand{\norm}[1]{\left\lVert #1 \right\rVert}
\newcommand{\inner}[2]{\left\langle #1,\, #2 \right\rangle}

\newcommand{\ip}[2]{\inner{#1}{#2}}
\newcommand{\F}{\mathrm{F}}
\newcommand{\NormOp}{\mathrm{Normalize}}
\newcommand{\Normalize}[1]{\NormOp\!\left(#1\right)}
\newcommand{\etaang}{\eta^{\mathrm{ang}}}
\newcommand{\wdc}{\lambda}
\newcommand{\lr}{\eta}

\title{Fantastic Pretraining Optimizers and Where to Find Them II:\\
Hyperball Optimization}

\newcommand{\stanfordaffil}{\textsuperscript{\(\dagger\)}}
\newcommand{\princetonaffil}{\textsuperscript{\(\ddagger\)}}
\newcommand{\tsinghuaaffil}{\textsuperscript{\S}}

\author{
Kaiyue Wen\stanfordaffil \quad
Xingyu Dang\princetonaffil \quad
Kaifeng Lyu\tsinghuaaffil \quad
Tengyu Ma\stanfordaffil \quad
Percy Liang\stanfordaffil\\
{\small
\texttt{kaiyuew@stanford.edu} \quad
\texttt{xingyu.dang@princeton.edu} \quad
\texttt{klyu@mail.tsinghua.edu.cn}}\\
{\small \texttt{tengyuma@stanford.edu} \quad
\texttt{pliang@cs.stanford.edu}}
}

\date{}

\begin{document}

\maketitle
\begingroup
\renewcommand{\thefootnote}{\fnsymbol{footnote}}
\footnotetext[2]{Stanford University. \quad
\princetonaffil Princeton University. \quad
\tsinghuaaffil Tsinghua University.}
\endgroup

\begin{abstract}
Matrix based optimizers such as Muon can substantially speed up language model pretraining, but their gains over AdamW are observed to shrink as model size and data scale grow when using standard constant decoupled weight decay.
We propose \textbf{Hyperball}, a simple optimizer wrapper that addresses this issue.
Given a base optimizer such as Adam or Muon, Hyperball sets the Frobenius norms of weight matrices and their corresponding optimizer updates to fixed constants.
On Qwen3 style models up to $1.2$B parameters, Muon Hyperball achieves $20$--$30\%$ token equivalent speedup over weight decay baselines.
Hyperball also improves learning rate transfer across widths and depths compared to decoupled weight decay.
This method is motivated by prior theory showing that training with weight decay leads to an equilibrium weight norm that only depends on the training hyperparameters. Through this mechanism, the weight decay then decides the angular learning rate, i.e. how fast the direction of the weight matrix changes.
\end{abstract}

\section{Introduction}
\label{sec:intro}

Previous work observed that the speedups of matrix based optimizers such as Muon~\citep{jordan2024muon,liu2025moonlight} over AdamW~\citep{loshchilov2019decoupled} shrink from roughly $30\%$ to about $10\%$ as model size and data scale grow~\citep{fantastic-optimizers-2024}.
This motivates a simple question: can we keep these optimizer speedups at higher compute?

We introduce \textbf{Hyperball} as a simple solution to the above question. Hyperball is an optimizer wrapper that enforces constant weight norms and update norms, transforming any base optimizer into its Hyperball variant.
The wrapper is motivated by prior theory on the role of weight decay in scale invariant layers and by the way most modern LLM training uses weight decay to control the size of the weights implicitly.
Let $W_t$ be the parameter matrix at step $t$, $u_t$ be the update provided by a base optimizer, $\lr_t$ be the learning rate, and $\wdc$ be the weight decay coefficient.
The standard decoupled weight decay update~\citep{loshchilov2019decoupled} applies
\[
W_{t+1} = (1-\lr_t\wdc)\,W_t - \lr_t\,u_t.
\]
Here $-\lr_t u_t$ adds new update information and typically increases the weight norm in the absence of weight decay.
The term $(1-\lr_t\wdc)W_t$ softly controls the norm by shrinking the weights toward zero every step.
In modern Transformer architectures with normalization layers~\citep{vaswani2017attention,xiong2020layer}, many weight matrices are modeled as scale invariant in the standard sense at the level of the loss: for a scalar $c>0$, rescaling one matrix leaves the loss unchanged, $L(cW) = L(W)$.
In this setting, weight decay is puzzling as classical $\ell_2$ regularization: if the loss is unchanged by the scale of $W$, penalizing $\norm{W}_\F$ cannot be the main reason it improves training.

Hyperball replaces this soft control on weight norm with an explicit constraint.
It decouples the magnitude of the weights from the direction of the update.
For a matrix $X$, the Frobenius norm $\norm{X}_\F = (\sum_{ij} X_{ij}^2)^{1/2}$ is the Euclidean norm of its entries.
Let $R=\norm{W_0}_\F$ be the initial Frobenius norm of the parameter matrix, and let $\Normalize{X}:=X/\norm{X}_\F$ be Frobenius normalization.
The Hyperball update is
\[
W_{t+1}
\;=\;
R\cdot
\Normalize{\Big(W_t - \lr_t\, R\cdot \Normalize{u_t}\Big)}.
\]
\begin{wrapfigure}{r}{0.46\textwidth}
  \centering
  \includegraphics[width=\linewidth]{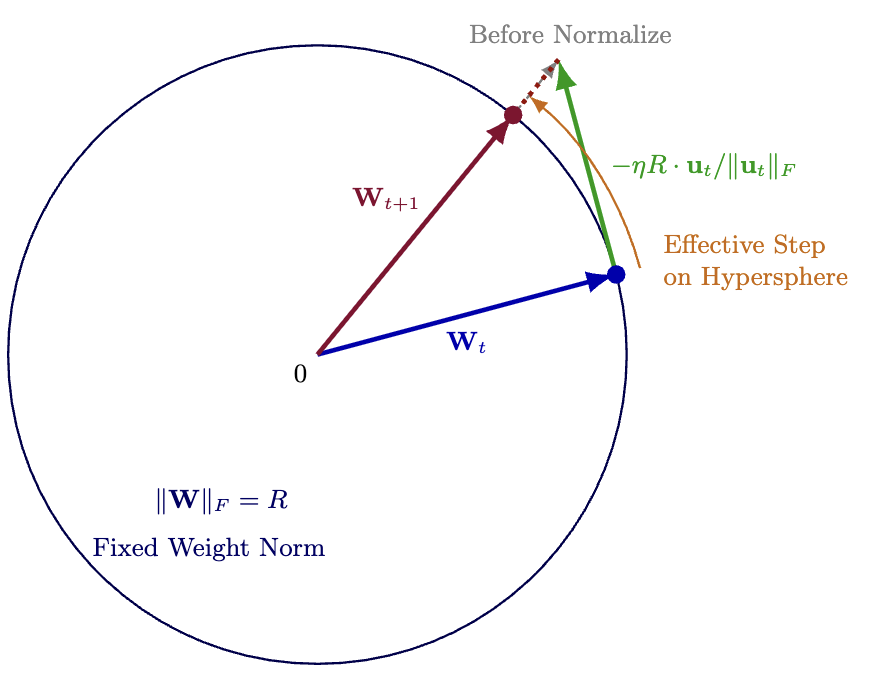}
  \caption{
    Geometric view of the Hyperball update.
    Weights are constrained to the sphere of radius $R$.
    Each step moves along the normalized update direction $-\Normalize{u_t}$ by a distance $\lr_t R$ and is immediately projected back to the sphere.
    The weight norm and update norm are held constant by construction.
  }
  \label{fig:hyperball-schematic}
\end{wrapfigure}
Geometrically, Hyperball constrains the optimization trajectory to the surface of a hypersphere with radius $R$.
The update takes a step of length $\lr_t R$ in the direction defined by the normalized update $-\Normalize{u_t}$, and the result is immediately projected back onto the sphere.
This keeps the norm of the weights and updates constant, so the optimizer navigates primarily through weight directions.

The base update $u_t$ can come from any optimizer. In this paper we focus on Adam Hyperball (AdamH) and Muon Hyperball (MuonH).
We apply Hyperball to Transformer weight matrices and use Adam for embeddings, normalization gains, and other parameters whose norm carries semantic information.
On $1.2$B parameter Qwen3 style models~\citep{yang2025qwen3}, MuonH achieves $20$--$30\%$ token equivalent speedup over its weight decay counterpart, whereas MuonW gives only about $10\%$ at this scale.
Across depth and width sweeps, Hyperball keeps the best learning rate window better than the baseline: the maximal drift is about $1.4\times$ for AdamH and MuonH, compared with $2$--$4\times$ for AdamW and MuonW baselines.

The optimization theory in \cref{sec:theory-optimization} explains why this explicit constraint matches the role that weight decay already plays in scale invariant layers.
Let $R_t=\norm{W_t}_\F$ be the Frobenius norm of the parameter matrix, and let $\widehat{W}_t=W_t/R_t$ be its direction.
The decomposition $W_t=R_t\widehat{W}_t$ separates radial norm dynamics from directional dynamics: to first order, the angular movement per step scales with the update norm divided by $R_t$.
Prior analyses of normalized networks and rotational equilibrium~\citep{li2020reconciling,roburin2020spherical,kosson2024rotational} show that, under a noise dominated model, decoupled weight decay balances stochastic norm growth and converges to an equilibrium radius.
Substituting this radius into the tangent dynamics yields an angular step size $\etaang$ that depends on the learning rate and weight decay mainly through the product $\lr\wdc$.
Hyperball uses this mechanism directly by fixing the radius and update norm, replacing the indirect calibration of $\wdc$ with an explicit angular learning rate schedule.

\begin{figure}[!t]
  \centering
  \includegraphics[width=\linewidth]{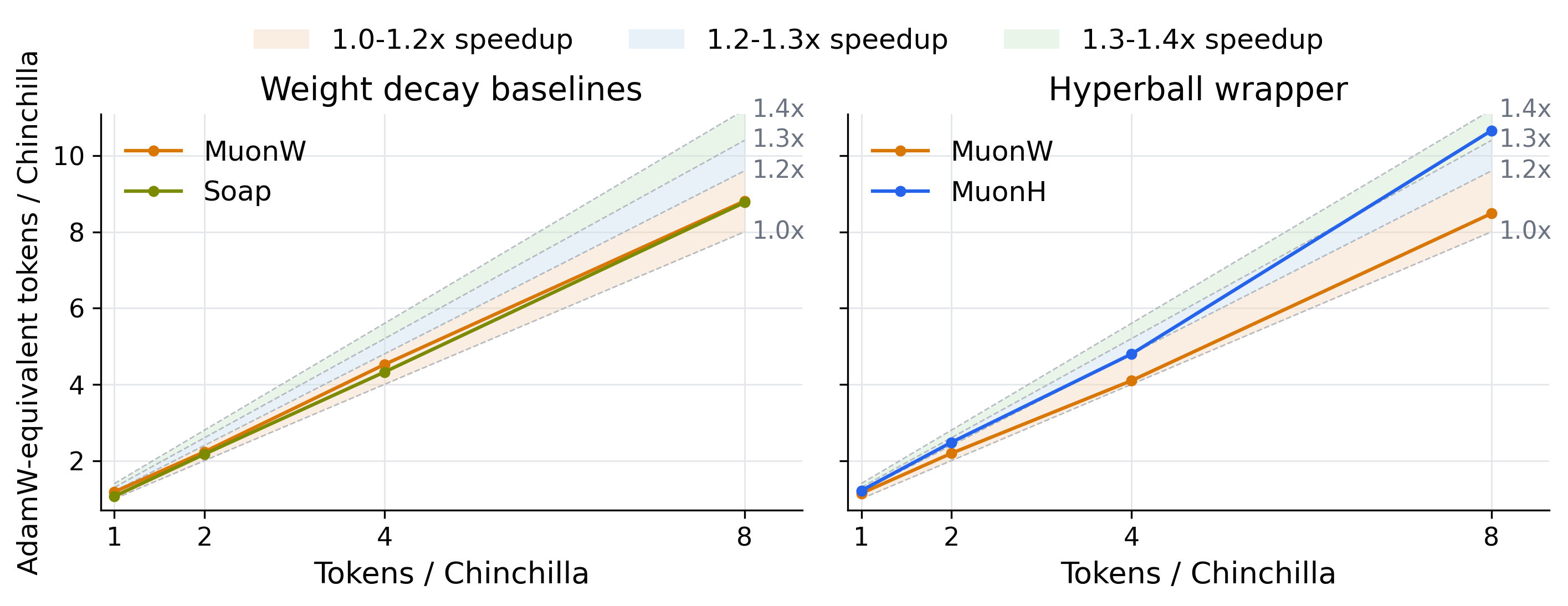}
  \caption{
    Token equivalent speedup over the AdamW scaling law at $1.2$B parameters across Chinchilla ratios $1\times$--$8\times$.
    The left panel is adapted from \citet{fantastic-optimizers-2024} and uses a setup without QK-Norm. The right panel uses the QK-Norm setup.
    In both setups, MuonW alone gains $\approx 10\%$ at this scale, while MuonH sustains $20$--$30\%$ speedup that grows with training duration.
  }
  \label{fig:speedup-1p2b}
\end{figure}

\section{Method}
\label{sec:method}

\paragraph{Definition.}
\label{sec:hyperball_def}

Let $W_t$ be the parameter matrix at step $t$, $u_t$ be the base optimizer update for this matrix (for example, Adam's preconditioned update~\citep{kingma2015adam} or Muon's matrix sign momentum update~\citep{jordan2024muon,liu2025moonlight}), $\lr_t$ be the Hyperball learning rate, $R>0$ be the fixed radius, and $\Normalize{X} := X/\norm{X}_\F$ be Frobenius normalization.
For each constrained matrix $W_0 \in \R^{d_{\mathrm{out}}\times d_{\mathrm{in}}}$, we initialize entries with standard deviation $1/\sqrt{d_{\mathrm{in}}}$ and set the radius once as $R=\norm{W_0}_\F$.
The Hyperball update is
\begin{equation}
\label{eq:hyperball-update}
W_{t+1} \;=\; R \cdot
\Normalize{\Big(W_t - \lr_t\, R \cdot \Normalize{u_t}\Big)}.
\end{equation}
Thus Hyperball takes an unconstrained optimizer step with norm $\lr_t R$, followed by radial renormalization to radius $R$ (\cref{alg:hyperball}, \cref{fig:hyperball-schematic}).
Write $\widehat{u}_t := u_t/\norm{u_t}_\F$ and $\widehat{W}_t := W_t/\norm{W_t}_\F$.
Equivalently, the exact displacement is
\begin{equation}
\label{eq:hyperball-displacement}
W_{t+1}-W_t
=
R\left(
\frac{\widehat{W}_t-\lr_t\widehat{u}_t}
{\norm{\widehat{W}_t-\lr_t\widehat{u}_t}_\F}
-\widehat{W}_t
\right),
\end{equation}
so radial renormalization returns the trial point to the sphere of radius $R$.

\begin{algorithm}[t]
\caption{Hyperball wrapper for a parameter matrix $W$}
\label{alg:hyperball}
\begin{algorithmic}[1]
\STATE {\bfseries Input:} parameter matrix $W_t$, base optimizer $\mathcal{O}$, optimizer state $\mathcal{S}_t$, radius $R$, schedule $\{\lr_t\}$
\STATE Compute base optimizer update $u_t, S_{t+1} \leftarrow \mathcal{O}(\nabla_{W_t} L(W_t), \mathcal{S}_t)$ 
\STATE Set normalized update direction $\widehat{u}_t \leftarrow \Normalize{u_t}$
\STATE Take unprojected step $\widetilde{W}_{t+1} \leftarrow W_t - \lr_t\,R\,\widehat{u}_t$
\STATE Project to radius $R$: $W_{t+1} \leftarrow R \cdot \Normalize{\widetilde{W}_{t+1}}$
\end{algorithmic}
\end{algorithm}

\paragraph{Where to apply the constraint.}
\label{sec:hyperball_where}

We apply Hyperball to attention and MLP weight matrices in a prenorm Transformer~\citep{vaswani2017attention,xiong2020layer}.
Embeddings, normalization gains, and other scalar parameters are updated with a standard optimizer (Adam in our experiments), since for these parameters the norm can carry semantic information.

\paragraph{Discussion of the design.}
\label{sec:hyperball_discussion}

A natural alternative to the Frobenius norm constraint is a spectral norm constraint.
Let $\norm{W}_{\mathrm{op}}$ denote the operator norm of a matrix.
The spectral condition of \citet{yang2023spectral} identifies relative spectral update size as a central quantity for feature learning, and SSO constrains training to the spectral sphere by steepest descent or projection~\citep{xie2026sso}.
In this view, a spectral Hyperball variant would fix $\norm{W_t}_{\mathrm{op}}$ and normalize the update in operator norm, directly controlling the sharpest layerwise scaling factor.

We use the Frobenius version in \eqref{eq:hyperball-update} for two reasons: computation cost and the theoretical motivation in \cref{sec:mechanism}.
Projection onto the Frobenius sphere is $O(N^2)$ per matrix, whereas exact spectral projection generally requires an SVD and costs $O(N^3)$.
One diagnostic for whether Frobenius control is close to spectral control is the stable rank ratio.
For a parameter matrix $W \in \mathbb{R}^{d_{\mathrm{out}}\times d_{\mathrm{in}}}$, define
\begin{equation}
\label{eq:stable-rank}
\mathcal{R}(W) \;:=\; \frac{\norm{W}_\F^2}{\norm{W}_{\mathrm{op}}^2} \;\in\; [1,\, \min(d_{\mathrm{in}}, d_{\mathrm{out}})].
\end{equation}
Values satisfying $\mathcal{R}(W) = \Omega(\min(d_{\mathrm{in}}, d_{\mathrm{out}}))$ indicate that the singular value spectrum is not dominated by a single direction, in which case Frobenius constraints behave similarly to spectral constraints up to a slowly varying factor.
The same high stable rank regime is observed empirically in the Kimi Moonlight analysis~\citep[Appendix~F]{liu2025moonlight}.
A spectral Hyperball variant is a natural direction when singular value concentration makes Frobenius control too loose.

\section{Hyperball Experiments}
\label{sec:experiments}

\subsection{Setup}
\label{sec:experiments_setup}

Unless noted otherwise, we use a Qwen3 style decoder only architecture~\citep{yang2025qwen3} with QK-Norm~\citep{henry2020query}, trained on a mixture of DCLM-baseline~\citep{li2024datacomp}, StarCoder~\citep{li2023starcoder}, and ProofPile~2~\citep{azerbayev2023llemma} (and FineWeb-Edu~\citep{penedo2024fineweb} for some runs).
We compare Adam~\citep{kingma2015adam} and Muon~\citep{jordan2024muon,liu2025moonlight} with decoupled weight decay (AdamW and MuonW) against their Hyperball variants AdamH and MuonH. 
For the speedup metric, we fit a scaling law to AdamW across Chinchilla ratios~\citep{hoffmann2022training} $\{1\times, 2\times, 4\times, 8\times\}$ and report, for each method's final loss, the token ratio $\tau = N_{\mathrm{AdamW}}/N_{\mathrm{method}}$ that AdamW would need to match it.
For learning rate transfer, we sweep a multiplicative grid (ratio $\sqrt{2}$) at each scale $s$, define $\lr^\star(s) := \arg\min_{\lr_k}\ \mathrm{ValLoss}(s, \lr_k; T)$, and report $\mathrm{Drift} := \max_s \lr^\star(s) \,/\, \min_s \lr^\star(s)$.

\subsection{End-to-end speedup}
\label{sec:experiments_speedup}

In the weight decay baseline setting adopted from \citet{fantastic-optimizers-2024}, $1.2$B parameter Qwen3 style models are trained over Chinchilla ratios $1\times$--$8\times$, and MuonW alone yields $\approx 10\%$ token equivalent speedup over the AdamW scaling law.
MuonH instead sustains $20$--$30\%$ speedup, and the gap \emph{grows} with training duration (\cref{fig:speedup-1p2b}).
Qualitatively, Hyperball starts slightly worse but overtakes WD as the learning rate decays.
On the Marin speedrun benchmark (FineWeb-Edu, $1\times$ Chinchilla), AdamH and MuonH match WD baselines that are $\approx 10\%$ larger (\cref{fig:marin-speedrun-ferries}, left).
In Marin Ferries, scaling MuonH to $8$B parameters yields a further $0.04$ loss improvement over the AdamW baseline, with both runs using manually chosen hyperparameters (\cref{fig:marin-speedrun-ferries}, middle and right).

\begin{figure}[t]
  \centering
  \includegraphics[width=\linewidth]{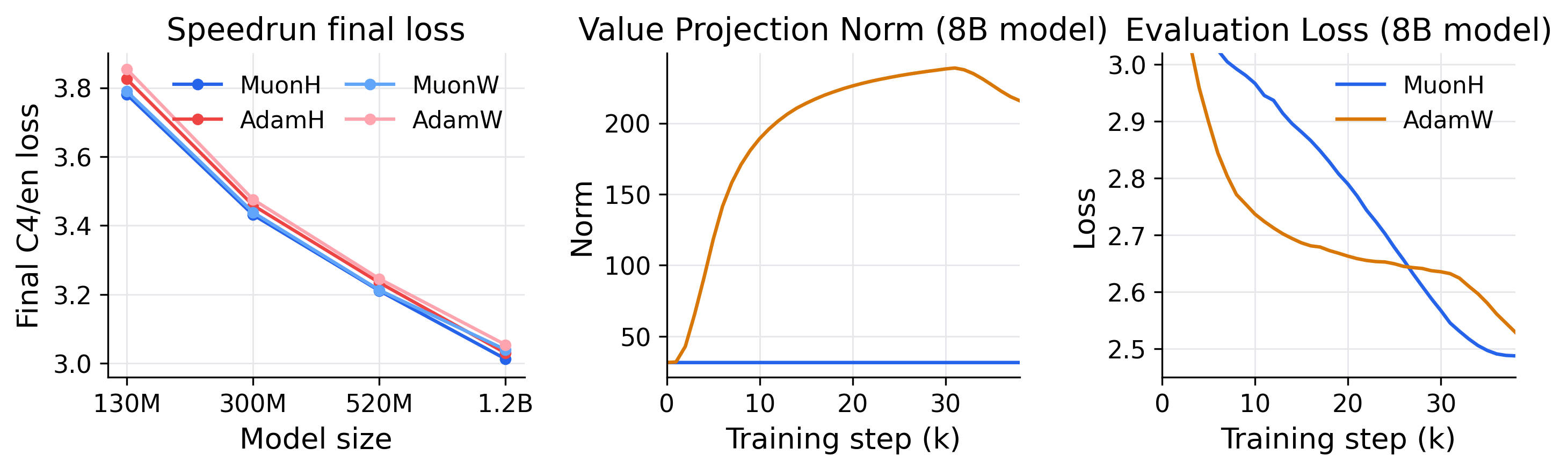}
  \caption{
    Additional Marin benchmarks.
    Left: final C4/en loss~\citep{raffel2020exploring} on the FineWeb-Edu speedrun benchmark at $1\times$ Chinchilla.
    Middle and right: $8$B model comparison over $159$B tokens. MuonH fixes the layer 9 value projection matrix norm and finishes $0.04$ lower than the AdamW baseline.
  }
  \label{fig:marin-speedrun-ferries}
\end{figure}

\begin{figure}[!htbp]
  \centering
  \includegraphics[width=0.92\linewidth]{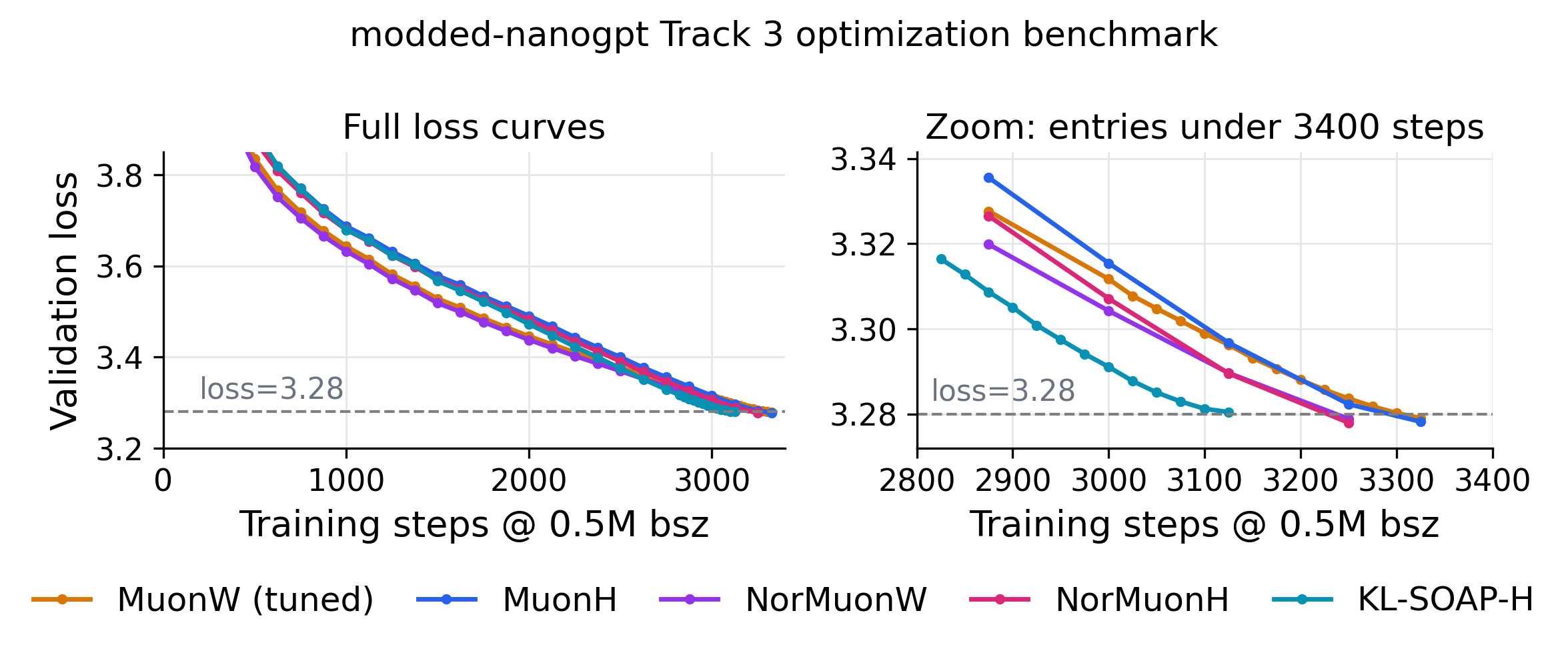}
  \caption{
    The modded-nanogpt Track~3 optimization benchmark.
    Curves show average validation loss from the public Track~3 logs versus training step. Lower and further left is better. Left: full trajectories. Right: zoom near the $3.27$--$3.28$ validation loss band for entries reaching this band within $3400$ steps.
    Hyperball variants improve the matched weight decay baselines in this comparison, and KL-SOAP-H, which denotes KL-SOAP~\citep{lin2026klshampoo} combined with Hyperball, reaches average validation loss $3.2780$ in $3125$ steps.
  }
  \label{fig:track3-optimization}
\end{figure}

On the public modded-nanogpt Track~3 optimization benchmark, which fixes the model and data and measures optimizer progress by step count, the corresponding WD baselines reach average validation loss $3.2790$ in $5625$ steps for the single run AdamW baseline, $3.2790$ in $3325$ steps for MuonW (tuned), and $3.2789$ in $3250$ steps for NorMuonW~\citep{li2025normuon}.
AdamH reaches average validation loss $3.2741$ in $4875$ steps. MuonH reaches average validation loss $3.2782$ in $3325$ steps, NorMuonH reaches average validation loss $3.2778$ in $3250$ steps, and KL-SOAP-H~\citep{lin2026klshampoo} reaches average validation loss $3.2780$ in $3125$ steps (\cref{fig:track3-optimization})~\citep{modded-nanogpt-track3}.
This result shows that Hyperball is not tied to Muon or Adam: paired with KL-SOAP, it reaches the fastest loss at step result in this comparison.

\subsection{Hyperparameter transfer}
\label{sec:experiments_transfer}

By construction, Hyperball fixes $\norm{W_t}_\F=R$ and uses unit Frobenius update directions, so $\lr_t$ directly sets the relative update length.
The optimal learning rate should therefore be approximately scale invariant.
We test this in two sweeps with $10$B tokens per run.
These transfer runs use a hybrid normalization architecture variant~\citep{zhuo2025hybridnorm}, with QK-Norm enabled.

\begin{figure}[!htbp]
  \centering
  \includegraphics[width=0.9\linewidth]{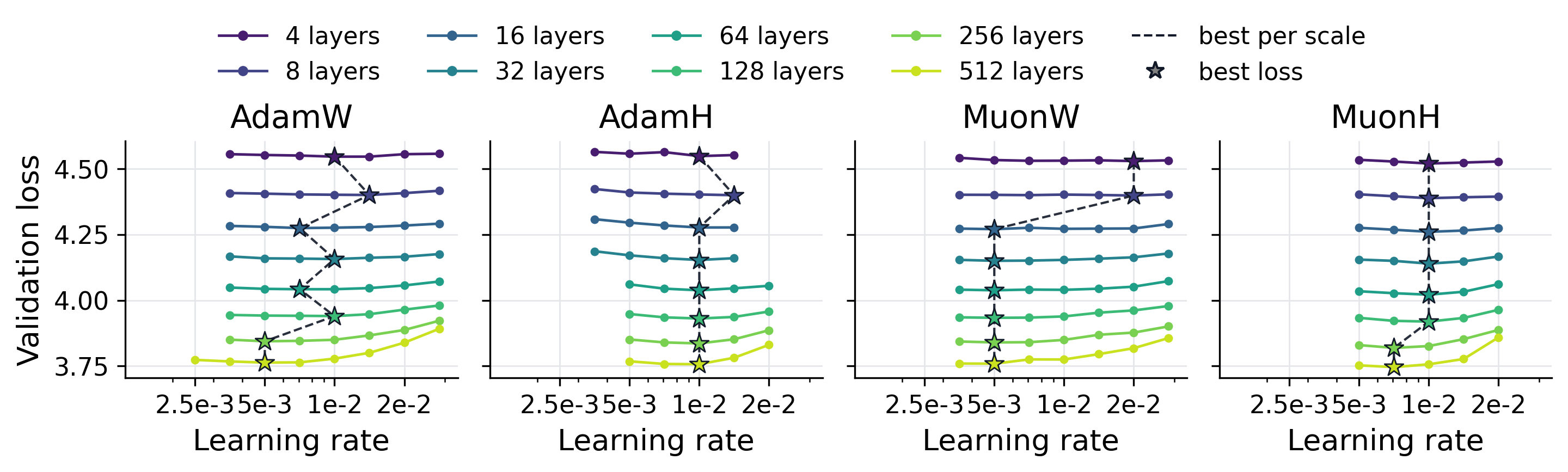}
  \caption{
    Depth scaling at fixed $d=128$, $10$B tokens per run.
    Each curve shows final validation loss versus learning rate for a given depth. Stars mark the best learning rate.
    Hyperball variants reduce the optimal learning rate drift across $L \in [4, 512]$ to $\approx 1.4\times$, versus $2$--$4\times$ for AdamW and MuonW.
  }
  \label{fig:depth-scaling}
\end{figure}

For depth scaling at fixed hidden dimension $d=128$ and $L \in \{4, \dots, 512\}$, the maximal drift of the optimal learning rate is $\approx 1.4\times$ for AdamH and MuonH, versus $\approx 3\times$ for AdamW and $\approx 4\times$ for MuonW even at $L=512$ (\cref{fig:depth-scaling}).

\begin{figure}[!htbp]
  \centering
  \includegraphics[width=0.9\linewidth]{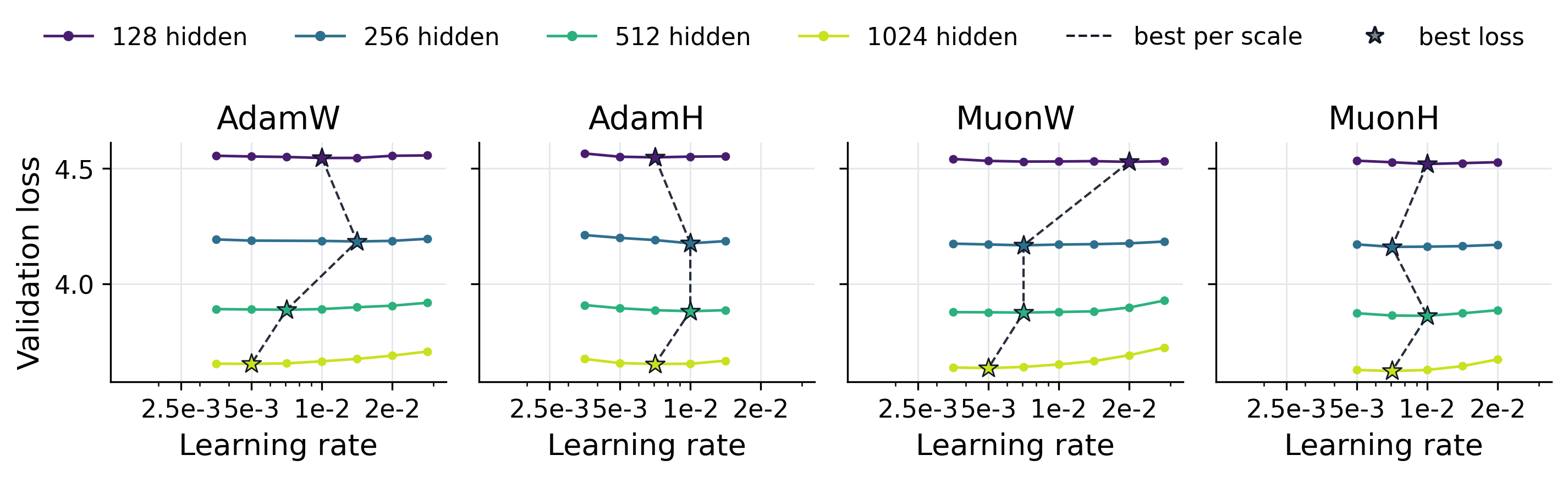}
  \caption{
    Width scaling at fixed $L=4$, $10$B tokens per run.
    Hyperball variants reduce the optimal learning rate drift across $d \in [128, 2048]$ to $\approx 1.4\times$.
  }
  \label{fig:width-scaling}
\end{figure}

For width scaling at fixed depth $L=4$ and $d \in \{128, \dots, 2048\}$, the same $\approx 1.4\times$ drift holds for both Hyperball variants (\cref{fig:width-scaling}).

\subsection{Overtrained Setting}
\label{sec:experiments_long_horizon}

We also test Hyperball in an overtrained data scaling setting.
For a $130$M parameter model, we train MuonW and MuonH over token budgets from $1$B to $128$B and sweep the learning rate at each budget.
Both MuonW and MuonH use the same hybrid normalization architecture variant as in previous section.
MuonH attains lower best C4 validation loss across the full range (\cref{fig:overtrained-130m}).
Fitting $L(N)=L_\infty + A N^{-\alpha}$ to the best loss curve gives $L_\infty=3.065$ for MuonH versus $3.079$ for MuonW.

\begin{figure}[!htbp]
  \centering
  \includegraphics[width=0.78\linewidth]{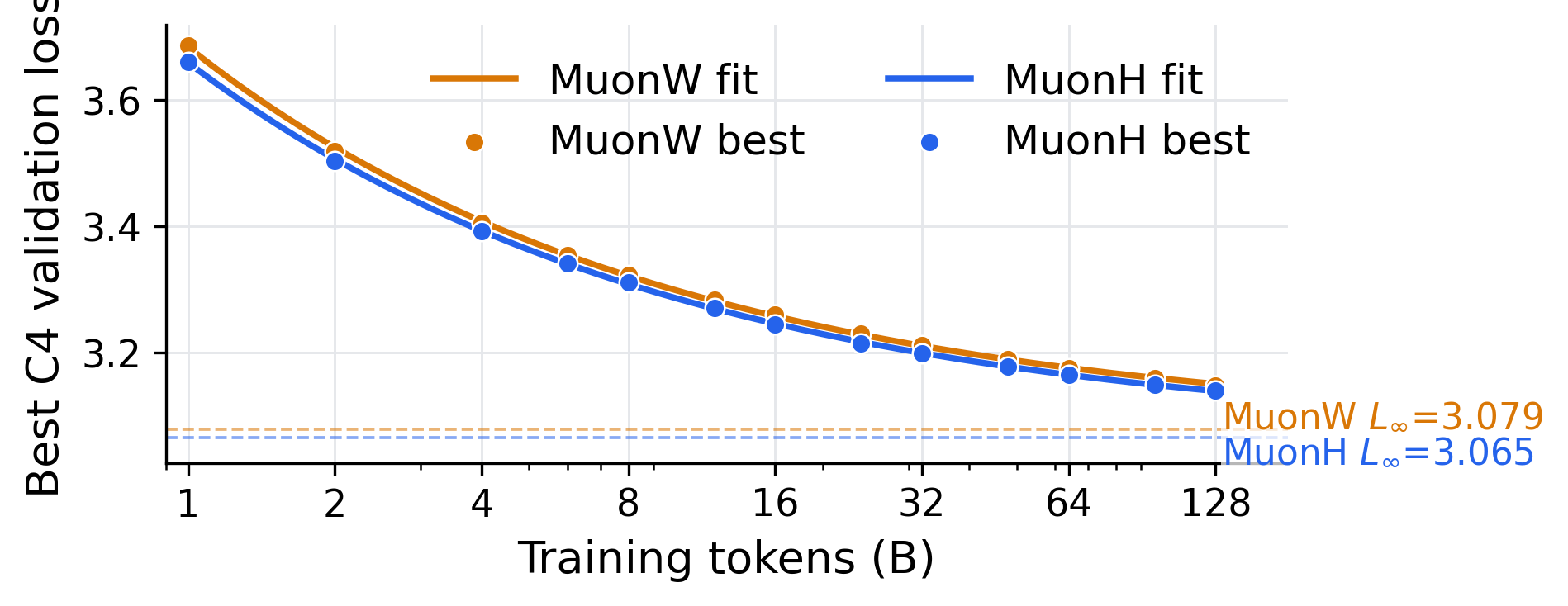}
  \caption{
    Overtrained $130$M data scaling sweep.
    Points show the best C4 validation loss over learning rate sweeps at each token budget.
    Solid curves fit $L(N)=L_\infty+A N^{-\alpha}$. Dashed lines mark the fitted asymptotes.
  }
  \label{fig:overtrained-130m}
\end{figure}

\FloatBarrier

\section{Theory}
\label{sec:mechanism}

\subsection{Expressivity}
\label{sec:theory-expressivity}

We first show that, in normalized networks, fixing the Frobenius norm of a weight matrix wouldn't limit representation power. This is because a trainable normalization gain can absorb the scale of the weight matrix, so Hyperball changes the optimization geometry without removing represented functions.
Let $h$ be the hidden state, $W$ be a weight matrix, and $\gamma$ be the RMSNorm gain~\citep{zhang2019root}.
For a a linear map placed after a layer norm,
\begin{equation}
\label{eq:prenorm-linear}
f(h;W,\gamma) \;=\; W\,(\gamma \odot \mathrm{RMSNorm}(h)).
\end{equation}
For any scalar $c>0$, the joint rescaling $(W,\gamma) \mapsto (cW, \gamma/c)$ leaves the represented function unchanged:
\begin{equation}
\label{eq:compensate}
f(h;cW,\gamma/c) = f(h;W,\gamma).
\end{equation}
Thus constraining $\norm{W}_\F$ need not reduce the represented function class when a trainable normalization gain can absorb the scale.

\subsection{Optimization}
\label{sec:theory-optimization}
With $W$ being a parameter block and $\mathcal{W}^c$ being the remaining parameters in the neural network. 
A loss $L(W, \mathcal{W}^c)$ is \emph{scale invariant} in a parameter block $W$ if $L(cW, \mathcal{W}^c) = L(W, \mathcal{W}^c)$ for all $c>0$. We will drop $\mathcal{W}^c$ and use the shorthand $L(W)$ from now on.
This is the notion used throughout the analysis: the radial coordinate $\norm{W}_\F$ is redundant, and the loss depends only on the direction $\widehat W = W/\norm{W}_\F$.
Many weight matrices in modern LLMs are only approximately scale invariant in this loss level sense, but the exact scale invariant model is a useful local approximation for the norm dynamics below.

\paragraph{Decoupled weight decay and angular motion.}
Let $W_t$ be the parameter matrix at step $t$, $u_t$ be the base optimizer update for this matrix, $\lr_t$ be the learning rate, $\wdc$ be the weight decay coefficient, and $\alpha_t := 1-\lr_t\wdc$.
The standard decoupled weight decay update~\citep{loshchilov2019decoupled} is
\begin{equation}
\label{eq:wd-update}
W_{t+1} = \alpha_t W_t - \lr_t u_t.
\end{equation}
Let $R_t := \norm{W_t}_\F$ and $\widehat W_t := W_t/R_t$.
The angular step size is the per step movement on the unit sphere,
\begin{equation}
\label{eq:etaang-def}
\etaang_t := \norm{\widehat W_{t+1}-\widehat W_t}_\F,
\end{equation}
and the relative update ratio is
\begin{equation}
\label{eq:rho-def}
\rho_t := \norm{u_t}_\F/\norm{W_t}_\F.
\end{equation}
For scale invariant $L$, the function represented by the block is determined by $\widehat W_t$, so $\etaang_t$ is the optimizer-controlled quantity that determines how fast this block moves in function space.
The direction after one decoupled step is exactly
\begin{equation}
\label{eq:direction-update}
\widehat W_{t+1}
=
\frac{\alpha_t\widehat W_t - \lr_t u_t/R_t}
{\norm{\alpha_t\widehat W_t - \lr_t u_t/R_t}_\F}.
\end{equation}
Thus, for fixed $\norm{u_t}_\F$, a larger radius $R_t$ gives a smaller angular movement.
Weight decay therefore acts as an indirect angular step controller by regulating $R_t$.

\paragraph{Concrete base updates.}
The analysis below uses AdamW, Muon, and Moonlight scaled Muon as examples.
Let $\ell_t$ be the minibatch loss, $g_t := \nabla_W\ell_t(W_t)$ be the stochastic gradient for the matrix block, let $\epsilon>0$ be Adam's numerical stability constant, and let divisions and square roots in Adam be elementwise.
For AdamW~\citep{loshchilov2019decoupled},
\begin{equation}
\label{eq:adamw-base-update}
m_t = \beta_1 m_{t-1} + (1-\beta_1)g_t,
\qquad
v_t = \beta_2 v_{t-1} + (1-\beta_2)g_t^{\odot 2},
\qquad
u_t = \frac{\widehat m_t}{\sqrt{\widehat v_t}+\epsilon},
\end{equation}
where $\widehat m_t$ and $\widehat v_t$ denote the bias corrected moments.
For Muon~\citep{jordan2024muon,liu2025moonlight}, let
\begin{equation}
\label{eq:muon-momentum}
M_t = \beta_1M_{t-1} + (1-\beta_1)g_t.
\end{equation}
For a matrix $A$ with compact singular value decomposition $A=P\Sigma Q^\top$, define the exact SVD matrix sign map by
\begin{equation}
\label{eq:matrix-sign-def}
\msign(A) := P Q^\top.
\end{equation}
Muon implementations often compute this map by Newton--Schulz iteration; in this theory we analyze the idealized SVD Muon update.
For $W\in\R^{d_{\mathrm{out}}\times d_{\mathrm{in}}}$, define $s_\mu:=\max(1,\sqrt{d_{\mathrm{out}}/d_{\mathrm{in}}})$.
The Muon base update is
\begin{equation}
\label{eq:muon-base-update}
u_t = s_\mu\,\msign(M_t).
\end{equation}
Moonlight~\citep{liu2025moonlight} uses the same momentum and exact SVD sign map, but replaces $s_\mu$ with $s_{\mathrm{moon}}:=0.2\sqrt{\max(d_{\mathrm{in}},d_{\mathrm{out}})}$:
\begin{equation}
\label{eq:moonlight-base-update}
u_t = s_{\mathrm{moon}}\,\msign(M_t).
\end{equation}

\paragraph{Idealized stationary model.}
We first make an assumption that assume we have an infinite history of gradient. This allows us to ignore boundary conditions on gradient when we consider momentum and weights. 

\begin{assumption}[Infinite history optimizer]
\label[assumption]{ass:infinite-history}
The gradient sequence $\{g_t\}_{t\in\mathbb Z}$ and the weight sequence $\{W_t\}_{t\in\mathbb Z}$ are defined for all integer times, including $t<0$.
Optimizer states are computed from this infinite past, i.e.,
\begin{equation}
m_t=(1-\beta_1)\sum_{i\ge0}\beta_1^i g_{t-i},
\qquad
M_t=(1-\beta_1)\sum_{i\ge0}\beta_1^i g_{t-i}.
\end{equation}
In the constant learning rate calculation below, $W_t$ is also taken to be the solution obtained by running \eqref{eq:wd-update} from the infinite past.
Intuitively, this describes the regime where training has already run for a long time, so the dependence on the initial optimizer state and initial weight has decayed.
\end{assumption}

We then make the following assumption on the distribution of $g_t$.

\begin{assumption}[Isotropic stationary gradients and idealized base maps]
\label[assumption]{ass:noise}
Let $p=d_{\mathrm{out}}$, $q=d_{\mathrm{in}}$, $d=pq$, and $r=\min(p,q)$.
After vectorizing the matrix block, the stochastic optimizer input is an iid isotropic Gaussian sequence,
\begin{equation}
\label{eq:noise-model}
g_t\sim\mathcal N(0,\sigma^2 I_d),
\qquad
\{g_t\}_{t\in\mathbb Z}\text{ independent.}
\end{equation}
\end{assumption}

\Cref{ass:noise} is intentionally idealized.
It ignores anisotropy, layer specific structure, and the signal in $\E[g_t]$.
Its purpose is to provide an easy to compute ansatz for update norms, update autocorrelations, radial equilibria, and angular step sizes.

For AdamW, we consider an idealized AdamW update by assuming that the second moment is correctly estimated.
For a scalar coordinate $\bar g_t$ with $\E[\bar g_t^2]=\sigma^2$, the stationary Adam second moment satisfies
\begin{equation}
\E[v_t]
=
(1-\beta_2)\sum_{i\ge0}\beta_2^i\,\E[\bar g_{t-i}^2]
=
(1-\beta_2)\sum_{i\ge0}\beta_2^i\,\sigma^2
=
\sigma^2.
\end{equation}
We therefore replace the Adam denominator by $\sigma$ and ignore the vanishing bias correction transient:
\begin{equation}
\label{eq:idealized-adam-update}
u_t = \frac{m_t}{\sigma},
\qquad
m_t=(1-\beta_1)\sum_{i\ge 0}\beta_1^i g_{t-i}.
\end{equation}

\paragraph{Update norm and autocorrelation.}

We will first study the correlation between updates at different steps, referred to as update autocorrelation.For Muon, the update autocorrelation is related to the following matrix sign map.

For $\rho\in[-1,1]$, define the SVD Muon sign kernel
\begin{equation}
\label{eq:muon-kernel}
\kappa_{p,q}(\rho)
:=
\frac{1}{\min(p,q)}\,\E\Big[\ip{\msign(X)}{\msign(\rho X + \sqrt{1-\rho^2}\,Z)}\Big],
\end{equation}
where $X,Z\in\R^{p\times q}$ have iid $\mathcal N(0,1)$ entries and are independent.
The normalization gives $\kappa_{p,q}(1)=1$ and $\kappa_{p,q}(0)=0$.

\begin{lemma}[Update norm and update autocorrelation]
\label{lem:update-norm}
Under \cref{ass:infinite-history,ass:noise}, define the update scale $U$ by
\begin{equation}
\label{eq:update-norm-main}
U =
\begin{cases}
\sqrt{\dfrac{1-\beta_1}{1+\beta_1}}\,\sqrt{pq} & \text{(idealized AdamW)},\\[0.7em]
\sqrt{p} & \text{(SVD Muon)},\\[0.25em]
0.2\sqrt{pq} & \text{(Moonlight scaled SVD Muon)},
\end{cases}
\end{equation}
Define the normalized autocorrelation sequence by $c_0=1$ and, for every lag $h\ge 1$,
\begin{equation}
\label{eq:optimizer-corrs}
c_h=
\begin{cases}
\beta_1^h & \text{(idealized AdamW)},\\[0.25em]
\kappa_{p,q}(\beta_1^h) & \text{(SVD Muon and Moonlight scaled SVD Muon)}.
\end{cases}
\end{equation}
For every lag $h\ge1$, these definitions give the second moment identities
\begin{equation}
\label{eq:update-moment-identities}
\E\norm{u_t}_\F^2=U^2,
\qquad
\E\ip{u_t}{u_{t-h}}=U^2c_h.
\end{equation}
\end{lemma}

\begin{proof}
For AdamW, each coordinate of \eqref{eq:idealized-adam-update} is a stationary Gaussian moving average.
If $\bar g_t$ is one coordinate, then
\begin{equation}
\bar u_t=(1-\beta_1)\sum_{i\ge0}\beta_1^i\frac{\bar g_{t-i}}{\sigma}.
\end{equation}
Thus $\E[\bar u_t^2]=(1-\beta_1)/(1+\beta_1)$.
For lag $h\ge1$,
\begin{equation}
\E[\bar u_t\bar u_{t-h}]
=
\frac{1-\beta_1}{1+\beta_1}\,\beta_1^h.
\end{equation}
Summing over $d=pq$ independent coordinates gives the AdamW line of \eqref{eq:update-norm-main} and both identities in \eqref{eq:update-moment-identities}, with $c_h=\beta_1^h$.

For SVD Muon, $\msign(A)$ has exactly $r=\min(p,q)$ nonzero singular values, all equal to $1$.
Therefore
\begin{equation}
\norm{s_\mu\msign(M_t)}_\F^2
=s_\mu^2 r
=\max(1,p/q)\min(p,q)
=p,
\end{equation}
which gives $U=\sqrt p$.
For Moonlight, the same calculation gives
\begin{equation}
\norm{s_{\mathrm{moon}}\msign(M_t)}_\F^2
=0.04\max(p,q)\min(p,q)
=0.04pq,
\end{equation}
so $U=0.2\sqrt{pq}$.

It remains to identify the autocorrelation.
The stationary momentum matrices satisfy
\begin{equation}
M_t=(1-\beta_1)\sum_{i\ge0}\beta_1^i g_{t-i}.
\end{equation}
Hence each pair $(M_t,M_{t-h})$ is jointly Gaussian, with identical marginal covariance and entrywise correlation $\beta_1^h$ using standard property of geometric sequences.
After dividing both matrices by their common standard deviation, the pair has the same distribution as
\begin{equation}
\bigl(X,\,\beta_1^h X + \sqrt{1-\beta_1^{2h}}\,Z\bigr),
\end{equation}
with $X,Z$ as in \eqref{eq:muon-kernel}.
Because the matrix sign is invariant to positive scalar rescaling, the normalized expected inner product of the two SVD Muon updates is exactly $\kappa_{p,q}(\beta_1^h)$.
Multiplying by $s_\mu^2r=U^2$ or $s_{\mathrm{moon}}^2r=U^2$ gives the autocorrelation identity in \eqref{eq:update-moment-identities}.
The scalar multiplier $s_\mu$ or $s_{\mathrm{moon}}$ cancels in the normalized autocorrelation, so Muon and Moonlight share the same $c_h$.
\end{proof}

\begin{wrapfigure}[13]{r}{0.40\textwidth}
  \vspace{-0.75em}
  \centering
  \includegraphics[width=\linewidth]{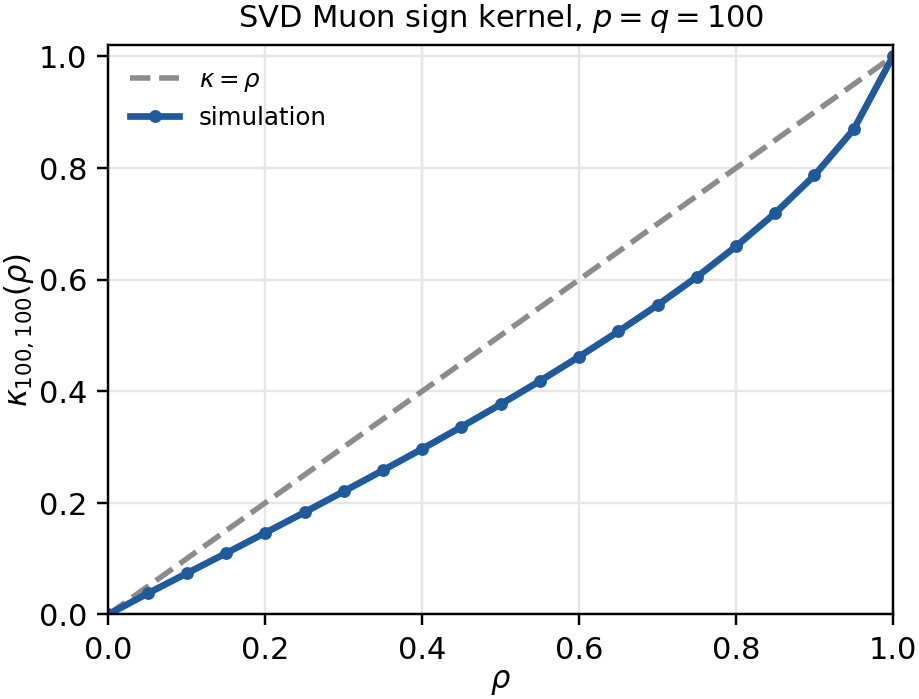}
  \vspace{-1.1em}
  \caption{
    Monte Carlo estimate of $\kappa_{p,q}(\rho)$ for $p=q=100$.
  }
  \label{fig:kappa-kernel-p100}
  \vspace{-0.75em}
\end{wrapfigure}
One question is what the range is for $\kappa_{p,q}$.
Let $F(X):=r^{-1/2}\msign(X)$.
Since $F(-X)=-F(X)$ and $X$ is symmetric, $\E[F(X)]=0$ and $\E\norm{F(X)}_\F^2=1$.
By the Hermite expansion of the Gaussian noise operator,
\begin{equation}
\label{eq:kappa-hermite}
\kappa_{p,q}(\rho)=\sum_{k\ge1} a_k\rho^k,
\qquad
 a_k\ge0,
\qquad
\sum_{k\ge1}a_k=1.
\end{equation}
Therefore, for $0\le\rho\le1$,
\begin{equation}
\label{eq:kappa-bound}
0\le\kappa_{p,q}(\rho)\le\rho.
\end{equation}

This inequality will be used later to show the correlation between weight and update is bounded.
\Cref{fig:kappa-kernel-p100} illustrates this range in a numerical simulation with $p=q=100$.

\paragraph{Weight update correlation.}
We will now consider constant hyperparameter training, the case where we use constant learning rate $\lr$ and constant weight decay $\wdc$.
We denote $1-\lr\wdc$ as $\alpha$, and assume $0<\alpha<1$.

\begin{lemma}[Stationary projection coefficient]
\label{lem:gamma-stationary}
Under \cref{ass:infinite-history,ass:noise} and constant hyperparameter training, let the normalized update autocorrelation sequence $c_h$ be given by \eqref{eq:optimizer-corrs}.
Let 
\begin{equation}
\label{eq:Calpha-def}
C_\alpha := \sum_{h=1}^{\infty}\alpha^{h-1}c_h.
\end{equation}

Define $\gamma_t:=\E\ip{W_t}{u_t}/U^2$ to be the projection coefficient, quantifying how strongly the weight and update correlates, then $\gamma_t$ identically equal to a constant value $\gamma$ for all $t$, with
\begin{equation}
\label{eq:gamma-main}
\gamma=-\lr C_\alpha.
\end{equation}
\end{lemma}

\begin{proof}
Since $0<\alpha<1$, the stationary solution of the decoupled weight decay recursion \eqref{eq:wd-update} is
\begin{equation}
\label{eq:stationary-w-unroll}
W_t=-\lr\sum_{h=1}^{\infty}\alpha^{h-1}u_{t-h}.
\end{equation}
Taking the expectation of the inner product with $u_t$ and using \eqref{eq:update-moment-identities} gives
\begin{equation}
\E\ip{W_t}{u_t}
=
-\lr\sum_{h=1}^{\infty}\alpha^{h-1}\E\ip{u_{t-h}}{u_t}
=
-\lr U^2\sum_{h=1}^{\infty}\alpha^{h-1}c_h.
\end{equation}
Dividing by $U^2$ and using the definition of $C_\alpha$ in \eqref{eq:Calpha-def} proves \eqref{eq:gamma-main}.
\end{proof}

For AdamW, $c_h=\beta_1^h$, and the sum can be simplified,
\begin{equation}
\label{eq:Calpha-adam}
\begin{aligned}
C_\alpha^{\mathrm{Adam}}
&=
\sum_{h=1}^{\infty}\alpha^{h-1}\beta_1^h \\
&=
\beta_1\sum_{h=1}^{\infty}(\alpha\beta_1)^{h-1}
=
\frac{\beta_1}{1-\alpha\beta_1}.
\end{aligned}
\end{equation}
For SVD Muon and Moonlight,
\begin{equation}
\label{eq:Calpha-muon}
C_\alpha^{\mathrm{Muon}}=C_\alpha^{\mathrm{Moonlight}}
=\sum_{h=1}^{\infty}\alpha^{h-1}\kappa_{p,q}(\beta_1^h).
\end{equation}
The series is finite because \eqref{eq:kappa-bound} gives $c_h\le\beta_1^h$.
The negative sign means that, in stationarity, the current update is negatively correlated with the current weight.
The size of this negative correlation is controlled by $C_\alpha$.
Substituting \eqref{eq:Calpha-adam} into \eqref{eq:gamma-main} recovers the familiar AdamW expression
\begin{equation}
\gamma
=
-\lr C_\alpha^{\mathrm{Adam}}
=
-\lr\frac{\beta_1}{1-\alpha\beta_1}
=
-\frac{\lr\beta_1}{1-\alpha\beta_1}.
\end{equation}

\paragraph{Equilibrium weight norm.}
Let stationary weight norm $S_t:=\E\norm{W_t}_\F^2$ and let $U^2$ be the update second moment in \eqref{eq:update-moment-identities}.
Squaring \eqref{eq:wd-update}, taking expectations, and substituting $\E\ip{W_t}{u_t}=\gamma U^2$ gives the following equality:
\begin{equation}
\label{eq:R-recursion-main}
S_{t+1}
=
\alpha^2 S_t+\lr^2U^2-2\alpha\lr\gamma U^2.
\end{equation}
At stationarity, $S_{t+1}=S_t=R_\star^2$, and \eqref{eq:gamma-main} gives
\begin{equation}
\label{eq:Rstar}
R_\star
=
\lr U\sqrt{\frac{1+2\alpha C_\alpha}{1-\alpha^2}}.
\end{equation}
For AdamW, \eqref{eq:Calpha-adam} reduces this to the previous closed form
\begin{equation}
\label{eq:Rstar-adam}
R_\star^{\mathrm{Adam}}
=
\lr U\sqrt{\frac{1+\alpha\beta_1}{(1-\alpha^2)(1-\alpha\beta_1)}}.
\end{equation}
For SVD Muon and Moonlight, the correct formula is instead \eqref{eq:Rstar} with $C_\alpha$ from \eqref{eq:Calpha-muon}.
In all cases, when $\lr\wdc$ is small and $\beta_1$ is fixed, $R_\star=\Theta(U\sqrt{\lr/\wdc})$ up to the autocorrelation factor $\sqrt{1+2\alpha C_\alpha}$.
Closely related estimates for AdamW update and weight RMS based on mean field approximation appear in \citet{su2025adamupdate,su2025adamwrms1,su2025adamwrms2}.

\begin{corollary}[Cosine and angular step at equilibrium]
\label{cor:cos-eq}
Under \cref{ass:infinite-history,ass:noise} and constant hyperparameter training, the following cosine proxy between the weight and update $\cos_t:=\frac{\E\ip{W_t}{u_t}}{R_\star U}$ identically equal to a constant value $\cos_\star$ for all $t$, with 
\begin{equation}
\label{eq:cos-main}
\cos_\star = -C_\alpha\sqrt{\frac{1-\alpha^2}{1+2\alpha C_\alpha}}.
\end{equation}
The corresponding ansatz angular step size at equilibrium is
\begin{equation}
\label{eq:etaang-general}
(\etaang)^2
=
\frac{2(1-\alpha)\bigl(1-(1-\alpha)C_\alpha\bigr)}{1+2\alpha C_\alpha}.
\end{equation}
For AdamW, this becomes
\begin{equation}
\label{eq:etaang_formula}
{\etaang}
=
\sqrt{\frac{2(1-\alpha)(1-\beta_1)}{1+\alpha\beta_1}}
=
\sqrt{\frac{2\lr\wdc(1-\beta_1)}{1+(1-\lr\wdc)\beta_1}}.
\end{equation}
\end{corollary}

\begin{proof}
The cosine formula follows from $\E\ip{W_t}{u_t}=\gamma U^2$, $\gamma=-\lr C_\alpha$, and \eqref{eq:Rstar}:
\begin{equation}
\cos_\star=\frac{\gamma U}{R_\star}
=-C_\alpha\sqrt{\frac{1-\alpha^2}{1+2\alpha C_\alpha}}.
\end{equation}
For the angular step, plug $\E\ip{W_t}{u_t}=\gamma U^2$ and $R_t=R_\star$ into \eqref{eq:direction-update}, and set $k_\star:=U/R_\star$.
This gives
\begin{equation}
\ip{\widehat W_{t+1}}{\widehat W_t}
=
\frac{\alpha-\lr\gamma k_\star^2}
{\sqrt{\alpha^2-2\alpha\lr\gamma k_\star^2+\lr^2k_\star^2}}.
\end{equation}
At equilibrium, the denominator equals $1$ by the stationary radius equation.
Using $k_\star^2=(1-\alpha^2)/(\lr^2(1+2\alpha C_\alpha))$ and $\gamma=-\lr C_\alpha$ gives
\begin{equation}
(\etaang)^2
=2-2\left(\alpha+\frac{C_\alpha(1-\alpha^2)}{1+2\alpha C_\alpha}\right)
=\frac{2(1-\alpha)\bigl(1-(1-\alpha)C_\alpha\bigr)}{1+2\alpha C_\alpha}.
\end{equation}
Substituting $C_\alpha=\beta_1/(1-\alpha\beta_1)$ gives \eqref{eq:etaang_formula}.
\end{proof}

\begin{theorem}[Weight decay sets equilibrium radius and angular step size]
\label{prop:etaang_stable}
Under \cref{ass:infinite-history,ass:noise} and constant hyperparameter training with $\alpha=1-\lr\wdc$, the weight in the decoupled weight decay update \eqref{eq:wd-update} will converge to the equilibrium radius \eqref{eq:Rstar}.
At this equilibrium, the angular step size is given by \eqref{eq:etaang-general}, with the AdamW specialization \eqref{eq:etaang_formula}.
The base optimizer enters through two quantities only: the update norm $U$ and the autocorrelation sum $C_\alpha$.
\end{theorem}

\begin{proof}
By \cref{lem:gamma-stationary}, $\E\ip{W_t}{u_t}=\gamma U^2$ with $\gamma=-\lr C_\alpha$.
The radial recursion \eqref{eq:R-recursion-main} then gives \eqref{eq:Rstar}, and \cref{cor:cos-eq} gives the angular step size.
\end{proof}

Under \cref{ass:infinite-history,ass:noise}, the optimizer enters through the identities in \eqref{eq:update-moment-identities}.
Once $U$ and $c_h$ are fixed, the weight decay recursion determines the stationary radius, cosine proxy, and angular step algebraically.

\Cref{tab:steady-state} consolidates the stationary quantities.
The update norm $U$ controls the equilibrium radius.
The autocorrelation sum $C_\alpha$ controls the momentum-induced radial correction, the cosine, and the angular step.
AdamW has $C_\alpha=\beta_1/(1-\alpha\beta_1)$, whereas SVD Muon and Moonlight use the matrix sign kernel in \eqref{eq:Calpha-muon}.

\begin{table}[t]
\centering
\small
\setlength{\tabcolsep}{6pt}
\renewcommand{\arraystretch}{1.5}
\begin{tabular}{@{\hspace{5pt}}l@{\hspace{8pt}}|@{\hspace{8pt}}ccc@{\hspace{5pt}}}
\toprule
Quantity & AdamW & SVD Muon & Moonlight SVD Muon \\
\midrule
Update norm $U$
  & $\sqrt{\tfrac{1-\beta_1}{1+\beta_1}}\,\sqrt{d_{\mathrm{in}}d_{\mathrm{out}}}$
  & $\sqrt{d_{\mathrm{out}}}$
  & $0.2\sqrt{d_{\mathrm{in}}d_{\mathrm{out}}}$ \\
Autocorr. sum $C_\alpha$
  & $\dfrac{\beta_1}{1-\alpha\beta_1}$
  & $\sum_{h\ge1}\alpha^{h-1}\kappa_{p,q}(\beta_1^h)$
  & $\sum_{h\ge1}\alpha^{h-1}\kappa_{p,q}(\beta_1^h)$ \\
Equilibrium norm $R_\star$
  & \multicolumn{3}{c}{$\lr U\sqrt{\dfrac{1+2\alpha C_\alpha}{1-\alpha^2}}$} \\
Cosine proxy $\cos_\star$
  & \multicolumn{3}{c}{$-C_\alpha\sqrt{\dfrac{1-\alpha^2}{1+2\alpha C_\alpha}}$} \\
Angular step size $(\etaang)^2$
  & \multicolumn{3}{c}{$\dfrac{2(1-\alpha)\bigl(1-(1-\alpha)C_\alpha\bigr)}{1+2\alpha C_\alpha}$} \\
\bottomrule
\end{tabular}
\caption{Steady state ansatz quantities for $\alpha:=1-\lr\wdc$, with $C_\alpha$ defined in \eqref{eq:Calpha-def}. AdamW uses the raw momentum autocorrelation. SVD Muon and Moonlight use the exact matrix sign autocorrelation kernel $\kappa_{p,q}$ defined in \eqref{eq:muon-kernel}, and therefore need not have the same angular step as AdamW.}
\label{tab:steady-state}
\end{table}

\begin{lemma}[Inverse gradient scaling for scale invariant losses]
\label{lem:gradient-scale}
Independently of \cref{ass:infinite-history,ass:noise}, if $L:\R^{d_{\mathrm{out}}\times d_{\mathrm{in}}}\to\R$ is differentiable and scale invariant, $L(cW)=L(W)$ for all $c>0$, then
\begin{equation}
\label{eq:grad-inverse}
\nabla_W L(cW)=\frac{1}{c}\nabla_W L(W)
\qquad\text{for all }c>0.
\end{equation}
\end{lemma}

\begin{proof}
Scale invariance gives $L(cW+c\epsilon)=L(W+\epsilon)$ for every matrix $\epsilon$.
Differentiating both sides with respect to $\epsilon$ at $\epsilon=0$ gives
\begin{equation}
\ip{\nabla_W L(cW)}{c\epsilon}=\ip{\nabla_W L(W)}{\epsilon}
\qquad\forall\epsilon,
\end{equation}
which forces $c\nabla_W L(cW)=\nabla_W L(W)$.
\end{proof}

Applying \cref{lem:gradient-scale} to $W_t=R_t\widehat W_t$ gives
\begin{equation}
\label{eq:grad-inverse-layer}
\norm{\nabla_W L(W_t)}_\F
=
\frac{1}{R_t}\norm{\nabla_W L(\widehat W_t)}_\F.
\end{equation}
Combined with $R_\star\propto\sqrt{\lr/\wdc}$ from \eqref{eq:Rstar}, this gives the scale invariant prediction $\norm{\nabla_W L(W_t)}_\F\propto\sqrt{\wdc/\lr_t}$ when the direction of the weight is fixed and the autocorrelation factor changes slowly.

\paragraph{Interpretation.}
The main message of the theory is that, for scale invariant matrix blocks, decoupled weight decay should be understood as an indirect controller of angular optimization speed rather than merely as a regularizer.
The preceding results make this mechanism explicit:

\begin{enumerate}
    \item For a scale invariant loss, the relevant optimization variable is the direction $\widehat W=W/\norm{W}_\F$, and the effective angular step is controlled by the angular learning rate $\etaang_t := \norm{\widehat W_{t+1}-\widehat W_t}_\F$.

    \item \Cref{lem:update-norm} shows that for common optimizer the update norm $U$ and the correlation for update between steps converge to constants that only depend on optimizer choices and hyperparameters.  

    \item \Cref{lem:gamma-stationary,prop:etaang_stable} show that (1) weight converges to an equilibrium radius $R_\star$ and (2) the angular learning rate converge to a constant $\etaang$ and both constant only depend on optimizer choices and hyperparameters.  

    \item \Cref{lem:gradient-scale} shows that, for scale invariant losses, larger equilibrium radius imply smaller gradient norms, giving the prediction $\norm{\nabla_W L(W_t)}_\F\propto 1/R_t$ when the direction remains unchange.
\end{enumerate}

Thus, weight decay has two coupled effects: it fixes the radial scale $R_\star$, and that radial scale determines the angular learning speed through $\norm{u_t}_\F/R_\star$.
This is the mechanism that Hyperball makes explicit: instead of letting weight decay indirectly determine both the matrix norm and the relative update length, Hyperball fixes the norm and the normalized update length directly.

\subsection{Empirical Validation}
\label{sec:theory-validation}

\paragraph{Phenomenon 1: weight norm tracks learning rate warmup and decay throughout training.}
Under a WSD learning rate schedule, \eqref{eq:Rstar} predicts that $R_t$ should rise during warmup and shrink during learning rate decay.
In a $1.2$B AdamW run with cosine learning rate decay, Q/K/V projection norms across layers show exactly this pattern: norms rise rapidly during warmup and then decrease during decay (\cref{fig:wd-qkv-diagnostics}, top).

\paragraph{Phenomenon 2: gradient norm increases through training.}
For scale invariant blocks, \eqref{eq:grad-inverse-layer} predicts that gradient norms scale approximately as $1/R_t$.
This phenomenon is also studied in \citet{defazio2025gradients}, where a similar explanation is provided.
In the same run, the corresponding Q/K/V gradient norms increase late in training as the weight norms shrink during learning rate decay (\cref{fig:wd-qkv-diagnostics}, bottom).

\begin{figure}[t]
  \centering
  \includegraphics[width=\linewidth]{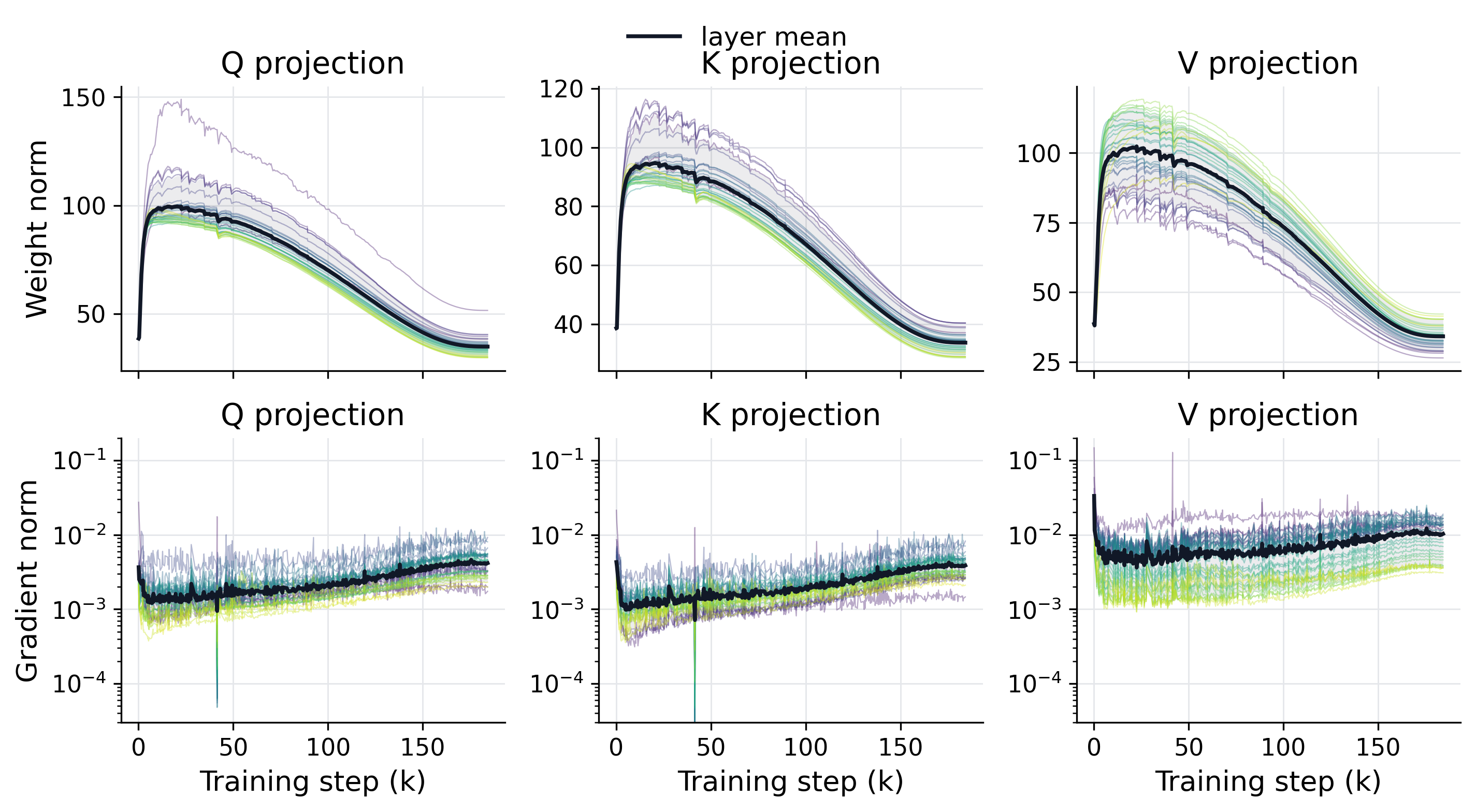}
  \caption{
    Weight norm and gradient norm diagnostics for a $1.2$B AdamW run.
    Top: Q/K/V projection weight norms across all $32$ layers.
    Bottom: the corresponding gradient norms.
    Thin curves are individual layers. The dark curve is the layer mean.
    Weight norms follow the learning rate schedule, and gradient norms rise as weight norms shrink during decay.
  }
  \label{fig:wd-qkv-diagnostics}
\end{figure}

\paragraph{Phenomenon 3: when $\lr\wdc$ is fixed, AdamW converges to essentially the same loss while each matrix norm is roughly proportional to $\lr$.}
\Cref{prop:etaang_stable} predicts that holding $\lr\wdc$ fixed keeps the angular step size $\etaang$ nearly fixed, so the training loss should be nearly unchanged.
Furthermore, if we divide $\wdc$ by $c$ and multiply $\lr$ by $c$, then \eqref{eq:Rstar} predicts $R_\star \propto c$.
In \cref{fig:wd-fixed-product-ablation}, we present two runs with $(\lr,\wdc)=(0.002,0.2)$ and $(0.004,0.1)$, and we observe that the train loss curves nearly overlap, while the equilibrium Q/K norms are roughly doubled in the larger learning rate run.

\begin{figure}[t]
  \centering
  \includegraphics[width=\linewidth]{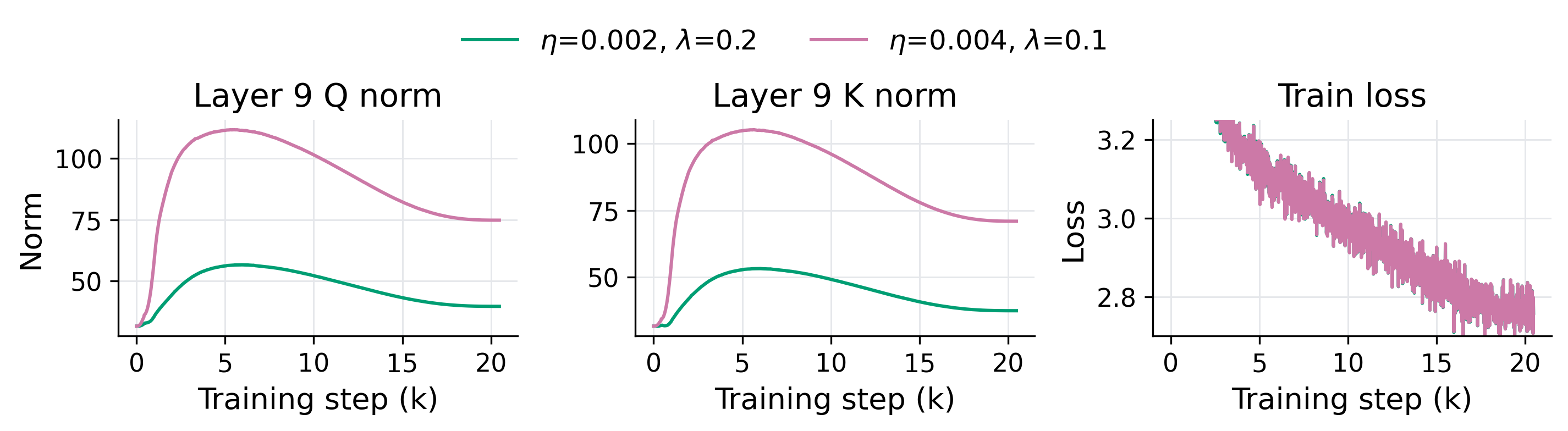}
  \caption{
    Ablation with fixed $\eta \lambda$.
    Two AdamW runs with the same product $\lr\wdc=4\cdot 10^{-4}$ have nearly identical train loss, while the larger learning rate run has roughly doubled layer 9 Q/K norms.
  }
  \label{fig:wd-fixed-product-ablation}
\end{figure}

\paragraph{Phenomenon 4: despite sharing the same learning rate schedule, weight decay starts with a higher loss but ultimately converges lower than no weight decay.}
When the WD and no WD runs use the same learning rate warmup and decay schedule, \eqref{eq:Rstar} and \eqref{eq:etaang_formula} predict different weight norms and angular step dynamics.
Without WD, the weight norm grows and the angular proxy $\lr\norm{u_t}_\F/\norm{W_t}_\F$ decays. With WD, the run maintains a larger effective step size throughout training.
Empirically, and in the theory in \cref{sec:theory-optimization}, training with WD yields a larger effective step size than training without WD.
According to the river valley theory~\citep{wen2024river}, the loss decomposes into a ``river'' component, capturing progress along a relatively flat direction where long term optimization happens, and a ``hill'' component, capturing excursions in steep directions caused by stochastic gradients.
A larger effective step size amplifies these hill direction oscillations, which raises the observed loss early in training, but it also accelerates motion along the river.
When the learning rate decays, the oscillations in the hill directions shrink and the iterate settles closer to the riverbed, revealing the additional progress that has already been made along the river.
This theory agrees with the phenomenon we observed here. The WD run starts with a higher loss but ultimately reaches a lower loss, because its larger effective step size allows it to move faster down the river before the decay phase suppresses the oscillations (\cref{fig:wd-vs-nowd-ablation}).

\begin{figure}[H]
  \centering
  \includegraphics[width=\linewidth]{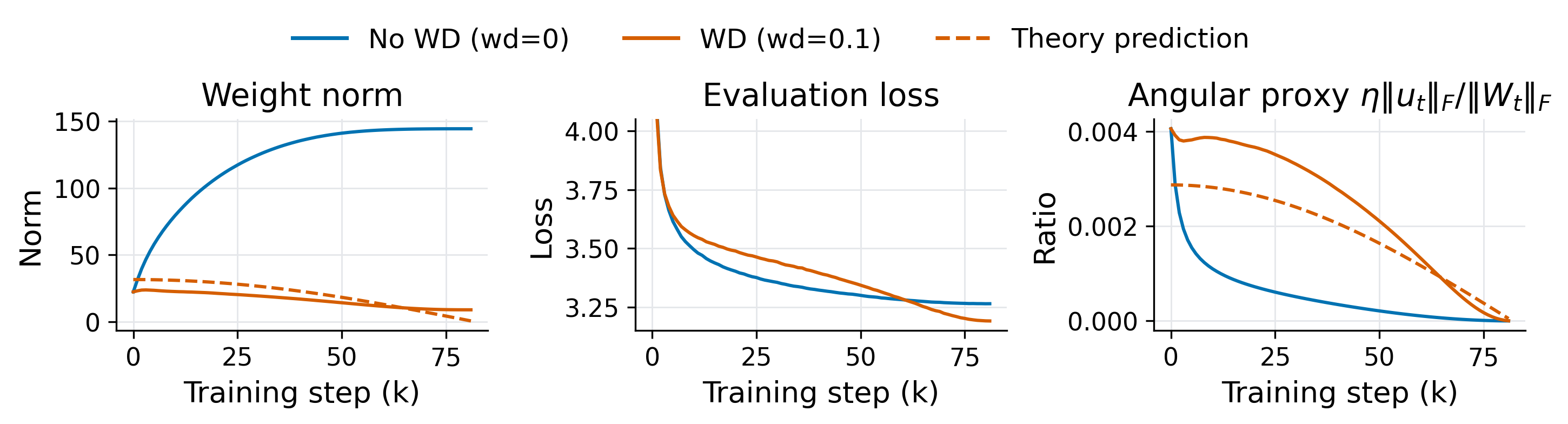}
  \caption{
    Weight decay versus no weight decay under the same learning rate schedule.
    Without weight decay, the weight norm grows and the angular proxy decays.
    With weight decay, the run maintains a larger late angular proxy and crosses over to a lower validation loss.
    The dashed curve is the theory prediction for the WD angular proxy.
  }
  \label{fig:wd-vs-nowd-ablation}
\end{figure}

\paragraph{Phenomenon 5: contrary to the original $\mu$P prediction, transfer is not sensitive to weight scale at initialization but is sensitive to weight decay scaling.}
Recent hyperparameter transfer studies find that transfer is often less sensitive to the initial weight scale than to how weight decay is scaled across model size and training duration~\citep{kosson2025weightdecay,blake2024umup,fan2025robust,wang2024adamw,qiu2025hyperparameter}.
This is consistent with \eqref{eq:Rstar}: the radial recursion forgets the initial radius and converges to a norm set by $\lr$, $\wdc$, and the optimizer dependent update norm $U$.
Changing the scaling rule for $\wdc$, however, changes both the equilibrium norm $R_\star$ and the angular step size in \eqref{eq:etaang_formula}, so it changes the dynamics relevant to transfer that are assumed by $\mu$P style analyses~\citep{yang2022tensorprograms}.
Hyperball turns this dependence into an explicit design choice by fixing the radius and normalized update length directly, which is why the same learning rate window transfers better across depths and widths in \cref{sec:experiments_transfer}.

\FloatBarrier

\section{Related work}
\label{sec:related}

\paragraph{Weight decay in normalized networks.}
Earlier work showed that, in normalized networks, weight decay often changes optimization dynamics or effective learning rates rather than acting as a classical capacity penalty~\citep{vanlaarhoven2017l2,zhang2019three,hoffer2018norm,dangelo2023why}.
A line of work argues that, in the presence of BatchNorm~\citep{ioffe2015batch} or LayerNorm~\citep{ba2016layer}, weight decay acts through norm dynamics: it sets an equilibrium weight norm and, jointly with the learning rate, an angular step size~\citep{li2020reconciling,roburin2020spherical,kosson2024rotational}.
\citet{yang2023spectral} formalize the role of the relative update size in feature learning at scale via a spectral condition.
Our analysis of \cref{sec:mechanism} sits squarely in this picture, closest in spirit to the rotational equilibrium framework of \citet{kosson2024rotational}. Hyperball is the matrix level wrapper that pins the relevant ratio directly rather than letting it equilibrate.

\paragraph{Norm constraints on weights and updates.}
Decoupling weight magnitude from direction has a long history.
Weight Normalization~\citep{salimans2016weight} reparameterizes $W=g\,V/\norm{V}$. Weight Standardization and BiT~\citep{qiao2019weight,kolesnikov2020big} standardize kernel statistics. Convolutional Normalization~\citep{liu2021convolutional} reduces per layer spectral norm. Decoupled Networks~\citep{liu2018decoupled} split feature norm from angle. Artificial Kuramoto Oscillatory Neurons~\citep{miyato2024akorn} use unit norm oscillator states.
On the update side, AdamP and SGDP~\citep{heo2021adamp} project updates onto the tangent space of the weight direction, Lion~\citep{chen2023lion} fixes the per entry update magnitude via the sign function, and LionAR normalizes early update sizes using angular criteria~\citep{kosson2024warmup}.
For generative models, EDM2~\citep{karras2023edm2} normalizes column weights and Spectral Normalization~\citep{miyato2018spectral} bounds the operator norm.
Hyperball differs in being an \emph{optimizer wrapper} that simultaneously fixes the matrix Frobenius norm and the update Frobenius norm, exposing the directional step size as a designed quantity.

\paragraph{Fixed norms in LLM pretraining and manifold optimization.}
Several recent and concurrent works enforce normalization at the architecture or optimizer level for language model pretraining.
nGPT~\citep{loshchilov2024ngpt} enforces columnwise unit norms with adaptive normalization layers. Nemotron-Flash~\citep{fu2025nemotron} applies per channel spherical constraints for inference time benefits but not on updates. The approximately normalized Transformer (anGPT)~\citep{franke2025compact} bounds each weight row using constrained parameter regularization~\citep{franke2024constrained}. \citet{owen2025variance} periodically rescale weights toward a target variance.
On richer manifolds, Modular Manifolds~\citep{modular-manifolds}, Muon$+$Stiefel~\citep{muon-stiefel}, notes on orthogonal manifolds and steepest descent~\citep{orthogonal-manifold,cesista2025stiefel}, and \citet{newhouse2025lipschitz} optimize on the Stiefel or spectral sphere manifold.
Related spectral norm views of Muon and weight decay appear in \citet{su2024spectral,chen2025muonspectral}, while SSO~\citep{xie2026sso} performs steepest descent or projection to the spectral sphere.
Hyperball projects onto the matrix Frobenius sphere $S^{d_{\mathrm{in}}d_{\mathrm{out}}-1}$---a softer constraint than normalization by column or channel---at $O(N^2)$ cost per matrix, versus $O(N^3)$ for spectral projections.

\section{Conclusion}
\label{sec:conclusion}

Hyperball replaces the implicit norm control of weight decay with an explicit optimizer constraint on matrix norms and update norms.
Across our experiments, the explicit constraint improves the scaling behavior of matrix based optimizers and makes learning rate transfer more reliable across model widths, depths, and training budgets.

A broader question is which constraint should be imposed.
The Frobenius norm is computationally cheap and theoretically motivated by the mechanism studied here, but spectral, rowwise, columnwise, hybrid, or architecture dependent constraints may better match some models and optimizers.
Another direction is to develop a sharper theory of weight normalized training, including Weight Normalization style parameterizations~\citep{salimans2016weight} and explicit norm constraints: when the radial degree of freedom is removed or fixed, how does this shape the training trajectory?

\section*{Acknowledgments}

Kaiyue Wen acknowledges support from the Stanford Graduate Fellowship.
Tengyu Ma acknowledges support from NSF grant 2522743.
This work was supported by the Google TPU Research Cloud (TRC), the Stanford HAI--Google Cloud Credits Program, the DARPA AIQ program, and NSF RI 2045685, and is part of the Marin Project.
The authors would like to thank Songlin Yang, Zihan Qiu, and Liliang Ren for motivating this project.
To some extent, this work is a proof of concept showing that it is possible to remove weight decay altogether by designing optimizers that explicitly control weight norms.
The authors would also like to thank William Held, David Hall, Suhas Kotha, Tatsunori Hashimoto, Jason Lee, Zhiyuan Li, Lijie Chen, Huaqing Zhang, Jiacheng You, Jeremy Bernstein, Shu Zhong, Samuel Schoenholz, Evan Walters, and Omead Pooladzandi for helpful discussions.

\bibliographystyle{abbrvnat}
\bibliography{hyperball}

\end{document}